\documentclass[twoside,11pt]{article}
\usepackage{subcaption}
\usepackage[preprint]{jmlr2e}

\usepackage{amsmath}
\usepackage{tabularx,booktabs}
\usepackage{rotating}

\newcommand{\mbb}[1]{\mathbb{#1}}
\newcommand{\mcal}[1]{\mathcal{#1}}

\newcommand{\R}{\mathbb{R}}
\newcommand{\Z}{\mathbb{Z}}
\newcommand{\N}{\mathbb{N}}
\newcommand{\rr}{\rightarrow}
\newcommand{\lrr}{\longrightarrow}

\newcommand{\on}[1]{\operatorname{#1}}

\usepackage{lastpage}
\jmlrheading{24}{2023}{1-\pageref{LastPage}}{3/23}{}{21-0000}{Taejin Paik}

\ShortHeadings{Invariant Representations of Embedded Simplicial Complexes}{Paik}
\firstpageno{1}

\begin{document}

% \title{Simplicial complex learning: Isometry and subdivision invariance}
\title{Invariant Representations of Embedded Simplicial Complexes}

\author{\name Taejin Paik \email paiktj@snu.ac.kr \\
       \addr Department of Mathematical Sciences\\
       Seoul National University\\
       Seoul, 08826, South Korea}

\editor{my editor}

\maketitle

\begin{abstract}%   <- trailing '%' for backward compatibility of .sty file
Analyzing embedded simplicial complexes, such as triangular meshes and graphs, is an important problem in many fields.
We propose a new approach for analyzing embedded simplicial complexes in a subdivision-invariant and isometry-invariant way using only topological and geometric information.
Our approach is based on creating and analyzing sufficient statistics and uses a graph neural network.
We demonstrate the effectiveness of our approach using a synthetic mesh data set. 
\end{abstract}


\section{Introduction}
Geometric objects such as triangular meshes and graphs are commonly used in various applications, including computer graphics, computer vision, and machine learning.
We use a simplicial complex as a mathematical model for such meshes and graphs, which provides a more general framework for analyzing topological spaces.
Analyzing these objects is often challenging, as they have complex structures that can be difficult to represent and process.
Recently, several approaches have been proposed to analyze such geometric objects, such as multi-view-based methods, voxel-based methods, and point cloud-based methods.
These methods have proven to be effective for various tasks, including classification, segmentation, and reconstruction.

However, these traditional approaches often have limitations.
For example, some methods may not be robust to subdivision of a simplicial complex, which can lead to different results when an object is subdivided, and some other methods may not be robust to the sampling of points on mesh data sets.
Additionally, these methods may not be isometry-invariant, which can lead to different results when the object is rotated or transformed.
These limitations have motivated the development of new approaches that are more robust and invariant under these types of transformations.

In this paper, we propose a novel approach to analyzing simplicial complexes using only topological and geometric data.
Our approach is designed to be both isometry-invariant and subdivision-invariant, ensuring that the results remain approximately consistent even when the object is rotated, translated, or subdivided.
To achieve this, we use a graph neural network to create an $\on{O}(3)$-equivariant operator, and we use the Euler curve transform which allows us to analyze the object using its topological information.
To the best of our knowledge, our approach is the first to be able to deal with triangular meshes directly and to achieve isometry-invariance and subdivision-invariance.

Our approach is a significant departure from traditional methods.
We mainly concentrate on using topological data of a simplicial complex to generate sufficient statistics that describe the properties of the underlying object.
By doing this, we can analyze the object using a simple and robust representation that is invariant to both subdivision and isometry transformations.

The remainder of the paper is structured as follows.
In Section \ref{sec: Prelim}, we provide background information on semialgebraic sets, simplicial complexes, the Euler curve, and isometries.
In Section \ref{sec: method}, we describe our proposed approach in detail, including the sufficient statistics that we use and the $\on{O}(3)$-equivariant operator via graph neural networks.
In Section \ref{sec: exp}, we present experimental results that demonstrate the effectiveness of our approach with a synthetic mesh data set\footnote{
For the reproducibility of our experiments, we release the source code of our experiments at \url{https://github.com/TJPaik/InvariantRepresentation}.}.
In Section \ref{sec: discussion}, we provide a discussion of our approach, show limitations and suggest directions for future work.
In Section Appendix \ref{appensix: sec: proof}, we provide the proofs of the propositions and theorems we use.

% Introduction
% \cite{hajij2022simplicial}

\section{Preliminaries}
\label{sec: Prelim}

In this section, we aim to provide a comprehensive overview of the foundational concepts necessary for this paper.
This includes a brief overview of semialgebraic sets, simplicial complexes, a definition of subdivision, a filtration, persistent homology, and isometry, as well as the establishment of relevant notations.

The validity of our approach is deeply rooted in the concept of the Euler characteristic.
In order to ensure that the Euler characteristic is well-defined, we first consider semialgebraic sets in the following subsection.

\subsection{Semialgebraic Set}
A semialgebraic set is a fundamental concept in algebraic geometry and real algebraic geometry.
A semialgebraic set is a subset of Euclidean space defined by a finite number of polynomial inequalities with real coefficients.
The notion of semialgebraic sets is a natural generalization of the notion of algebraic sets, which are defined by polynomial equations.
Unlike algebraic sets, semialgebraic sets can have a nonempty interior, which makes them a more flexible tool in many applications.

Formally, a \textbf{semialgebraic set} is a subset of $n$-dimensional Euclidean space that can be expressed as a finite union or intersection of sets of two types: 
\begin{enumerate}
    \item sets defined by polynomial inequalities of the form $\{\bar{x}\in\R^n: f(\bar{x})>0\}$ and
    \item sets defined by polynomial equations of the form $\{\bar{x}\in\R^n: g(\bar{x})=0\}$,
\end{enumerate}
where $f$ and $g$ are polynomials in $\bar{x}= (x_1, \dots, x_n)$ with real coefficients.

We note that the class of semialgebraic sets is an important example of an o-minimal structure.
One of the main properties of o-minimal structures is that they admit a well-defined notion of the Euler characteristic, which is a topological invariant that measures the ``shape'' or ``structure'' of a space.
That is, for each semialgebraic set $S\subseteq \R^n$, there is a finite partition into cells, and the \textbf{Euler characteristic} with respect to the partition $\mcal{C}$ is defined as 
$$
\chi_{\mathcal{C}}(S):=k_0-k_1+\dots+(-1)^d k_d+\dots+(-1)^n k_n
$$
where $k_d$ is the number of $d$-dimensional cells in $\mathcal{C}$.
Not surprisingly, as stated in \citet{van1998tame}, the Euler characteristic does not depend on the partition:
\begin{proposition}
    For a semialgebraic set $A\subseteq\R^n$, assume that we have two finite partitions $\mcal{C}$ and $\mcal{C}^\prime$.
    Then, we have
    $$
    \chi_{\mathcal{C}}(S)=\chi_{\mathcal{C}^{\prime}}(S).
    $$
\end{proposition}
Therefore, we can simplify the notation and refer to the Euler characteristic of a semialgebraic set $S\subseteq \R^n$ by simply writing $\chi(S)$.
For a more comprehensive understanding of the concept, refer to \citet{van1998tame}.

In the following, we take a brief look at certain concepts from the Euler calculus that is useful later on.
Let $X$ be a semialgebraic set in $\mbb{R}^n$.
We call an integer-valued function $f:X\rightarrow \mbb{Z}$ \textbf{constructible} if $f^{-1}(i)$ is semialgebraic subset for every $i\in \mbb{Z}$.
We denote the set of bounded compactly supported constructible functions on $X$ as $\operatorname{CF}(X)$.
One of the essential features of the Euler characteristic is its additivity property, which makes it behave like a measure \citep{curry2012euler}.
That is, for semialgebraic sets $A$ and $B$, we have
$$
\chi(A\cup B) = \chi(A) + \chi(b) - \chi(A\cap B).
$$
This additivity is used to define the \textbf{Euler integral} for functions on $\operatorname{CF}(X)$; for every $f \in \operatorname{CF}(X)$, we define
$$
\int_X f\; d\chi := \sum_{i\in \Z} i\cdot \chi(f^{-1}(i)).
$$
The Euler integral is a fundamental component of the Euler calculus.
This technique helps solve problems related to counting and enumeration in computational geometry and sensor networks.
For a more detailed description of the Euler integral and the Euler calculus, see \citet{baryshnikov2011inversion, curry2012euler, mccrory1997algebraically}.

\subsection{Embedded Simplicial Complex}
In this section, we present an introduction to fundamental notions of a simplicial complex from \citet{edelsbrunner2022computational}.

\begin{definition}\label{def: simplicial complex}
A $k$-\textbf{simplex} is a convex hull of $k+1$ affinely independent points in an ambient space $\R^n$.
A \textbf{simplicial complex} $K$ is defined as a finite collection of simplices that satisfy:
\begin{enumerate}
    \item If $\tau$ is a face of $\sigma$ and $\sigma\in K$, then $\tau\in K$.
    \item Assume that $\sigma_0$ and $\sigma_1$ are elements of $K$. Then, $\sigma_0 \cap \sigma_1$ is a face of $\sigma_0$ and $\sigma_1$ if it is not the empty set.
\end{enumerate}
\end{definition}

In this paper, we primarily focus on \textbf{embedded simplicial complex}, which is defined as a union of simplices in a given simplicial complex $K$ with the subspace topology inherited from the ambient Euclidean space.
Essentially, an embedded simplicial complex is a simplicial complex that is embedded in a Euclidean space.
To avoid cumbersome notation, we refer to the embedded simplicial complex simply as $K$, using an abuse of notation.
Note that every simplex is a semialgebraic set, and therefore, embedded simplicial complexes are semialgebraic sets.
Since an embedded simplicial complex $K$ in $\R^n$ can be understood as a union of simplices, which is a semialgebraic cell in $\R^n$, the Euler characteristic of the embedded simplicial complex is
$$
\chi(K) = \sum_{m \geq 0}(-1)^m k_m
$$
where $k_m$ is the number of $m$-simplices.

Here, we introduce some notations for ease of use in later discussions.
For an embedded simplicial complex $K$ in $\R^n$, we write $MK$ to denote $\{M^{-1}x \mid x\in K\}$ for $M\in \operatorname{GL}_n(\R)$ for the convenience of notation later.
Also, for a vector $v$ in $\mbb{R}^n$, we denote the set $\{x + v\mid x\in K\}$ as $K + v$. 
These notations are helpful in making our discussions and equations more concise and easier to understand.

Let $K$ be an embedded simplicial complex.
An embedded simplicial complex $K^\prime$ is called \textbf{subdivision} of $K$ If
\begin{enumerate}
    \item each simplex of $K^\prime$ is contained in a simplex of $K$ and
    \item each simplex of $K$ is finite union of simplices of $K^\prime$.
\end{enumerate}
In other words, $K$ and its subdivision are the same as sets in the ambient space of $K$, but their simplicial structures can be different.

\subsection{Filtration}
A filtration is a powerful tool in the study of topological spaces.
The basic idea is to parameterize a space in a way that allows us to track changes in topology as the parameter changes.
In the context of this paper, filtrations are used to study the persistent homology and the Euler characteristic information of a space, which is a way to detect and quantify topological features.
Formally, a \textbf{filtration} is an indexed family of topological spaces $(X_i)_{i\in I}$ with a totally ordered index set $I$ such that $X_i \subseteq X_j$ for every $i\leq j$.
As we move forward, we provide an example of a filtration that will be of central focus in this work:
let $K$ be an embedded simplicial complex in $\R^n$.
For each unit vector $v$ in the unit sphere $S^{n-1}$ in $\mbb{R}^n$ and each real number $r$, we define 
$$
K_{v, r}:=\left\{x \in \R^n \mid x \cdot v \leq r\right\} \cap K.
$$
It follows that if $r_1\leq r_2$, then $K_{v, r_1} \subseteq K_{v, r_2}$, making $(K_{v, r})_{r\in\R}$ a filtration induced by the unit vector $v$.
Additionally, the set $\{x\in \mathbb{R}^n \mid x\cdot v \leq r\}$ is a semialgebraic set, and the class of semialgebraic sets is closed under finite intersections.
Therefore, the Euler characteristic of $K_{v, r}$ is well-defined.
Consequently, we obtain a curve $r\mapsto \chi(K_{v,r})$ for each unit vector $v$, and we call this curve the \textbf{Euler curve}.

Now we briefly introduce persistent homology, which is a powerful tool used to extract topological features from a filtration.
Suppose we have a filtration $(X_i)_{i\in \mbb{R}}$, where $X_i$ is a topological space for each index $i$.
We can then apply the homology functor $H_*$ with a finite field coefficient $\Z_p$ to obtain the homology groups $H_*(X_i;\Z_p)$ for each index $i$.
If each of these homology groups is finite-dimensional, then the persistence module $(H_*(X_i; \Z_p))_{i\in\mbb{R}}$ can be decomposed into a direct sum of interval modules \citep{oudot2017persistence, chazal2016structure}.

An interval module is represented by an interval of the form $(b, d)$, $[b, d)$, $(b, d]$, or $[b, d]$, where $b,d\in \overline{\R}$ where $\overline{\R}$ is the extended real line $[-\infty, \infty]$.
\textbf{Persistence diagram} is a compact representation of the topological features that persist throughout the filtration, and it allows us to analyze the evolution of the homology classes in a concise and intuitive manner.
To obtain the persistence diagram, we collect all the intervals and represent them as a multi-set of points on the (extended) coordinate plane.
The $x$-coordinate of each point corresponds to the birth time of a homology class, while the $y$-coordinate corresponds to its death time.
If an interval has infinite length, then we say that the corresponding entry in the persistence diagram has infinite persistence or infinite lifetime.

Persistence landscape is another representation of a persistence diagram.
It was introduced by \citet{bubenik2015statistical} as a way to vectorize the persistence diagram.
In contrast to the persistence diagram, which is a collection of points in the plane, the persistence landscape is a collection of functions defined on $\R$:
\begin{definition}
\label{def: PL}
Let $(X_i)_{i\in \R}$ be a filtration.
The \textbf{persistence landscape} is a sequence of functions $\{\lambda_c\}_{c\in\N}$ such that
\begin{align*}
    \lambda_c : \R&\longrightarrow\overline{\R}\\
    t&\longmapsto\sup \left\{m \geq 0 \mid \beta^{t-m, t+m} \geq c\right\},
\end{align*}
where $\beta^{a, b}=\dim(\on{Im}(X_a \rightarrow X_b))$.
\end{definition}
The persistence landscape has been proven to be an effective tool in various applications of topological data analysis.
One notable result in this area is the theorem that shows the invertibility of this conversion under certain conditions \citep{bubenik2020persistence}.
Specifically, if we restrict our attention to persistence diagrams with finitely many entries and there is no entry with infinite persistence, then we can obtain a unique persistence landscape for a given diagram.
The importance of this result stems from the fact that we do not lose any information through the conversion from persistence diagrams to persistence landscapes.

The approach of obtaining such homology features from a filtration, or sometimes the homology features itself is commonly referred to as a \textbf{persistent homology}.
To get a persistence homology, we usually need to choose a finite field for a coefficient.
In this paper, if we do not provide any specific field, it is implied that a finite field $\Z_p$ is being used for a prime $p$.

For our case, we deal with a filtration $\left(K_{v, r}\right)_{r \in \mathbb{R}}$ for an embedded simplicial complex $K$ and a unit vector $v$.
Therefore, we have persistence landscapes for each unit vector.
We denote the corresponding persistence landscape as $\{\lambda_{v, c}^K\}_{c\in \N}$.

Note that if there are finitely many entries in a persistence diagram with finite persistence, then the corresponding persistence landscape can be formalized using a different formula.
Specifically, Given a filtration, assume that the corresponding persistence diagram is a multi-set $\{(b_i, d_i)\}_{i\in I}$ where $I$ is a finite set, and $|b_i|$ and $|d_i|$ are finite for all $i\in I$.
Then, the persistence landscape can be defined via the formula
\begin{equation}\label{eq: pl_different rep}
\lambda_c(t)=\operatorname{cmax}\left\{f_{\left(b_{i}, d_i\right)}(t)\right\}_{i \in I}
\end{equation}
where $f_{(a, b)}(t)=\max (0, \min (a+t, b-t))$ and $\operatorname{cmax}$ denotes the $c$-th largest element \citep{bubenik2020persistence}.

\subsubsection{Computing the Euler Characteristic}
In this paper, calculating the Euler characteristic of an embedded simplicial complex is an important task.
One way to do this is by counting the number of simplices for each dimension, which is a straightforward approach.
Other than that, an alternative method for computing the Euler characteristic is by using the alternating sum of Betti numbers, which are topological invariants that measure the number of holes in the complex.
Interestingly, it is possible to obtain the Euler characteristic without having to determine all the Betti numbers explicitly using persistent homology:
\begin{theorem}\label{thm: finite field persistent}
    If $K$ is an embedded simplicial complex in $\mbb{R}^n$, then its Euler characteristic is independent of the base ring $G$. That is, 
    $$
    \chi(K) = \sum_{i=0}^n (-1)^i \operatorname{rk}_G H_i(K;G)
    $$
    where $\operatorname{rk}_G$ is the rank over the base ring $G$.
\end{theorem}
For an embedded simplicial complex $K$, most TDA (topological data analysis) software tools compute the persistent homology using a finite field coefficient $\Z_p$ for a prime $p$ \citep{maria2014gudhi,tauzin2021giotto,bauer2021ripser, bauer2017phat, tausz2012javaplex}.
Assuming that we have obtained the persistent homology for each dimension, we can calculate the rank of $H_*(K_{v, r};\Z_p)$ for each filtration value $r$.
However, this rank is not necessarily the same as the Betti number.
Nevertheless, from Theorem \ref{thm: finite field persistent}, we can obtain the Euler characteristic for each filtration value using the persistent homology.
The advantage of using persistent homology in this context is that it can provide a more comprehensive and robust analysis of the topological structure of the complex.
For our case, we can calculate the Euler curve $r\mapsto \chi(K_{v,r})$ using the persistent homology $(H_*(K_{v, r};\Z_p))_{r\in \R}$ for each $v$.

\begin{remark}
The Euler characteristic can depend on the base ring if $X$ is a more general topological space. For example, consider $X=\mbb{RP}^\infty$.
Its integer homology and homology with $\mbb{Z}_2$-coefficients are given by
\[
H_*(\mbb{RP}^\infty;\mbb{Z})
=\begin{cases}
\mbb{Z} & *=0 \\
\mbb{Z}_2& *>0 \text{ odd } \\
0 & \text{otherwise}
\end{cases}
\text{,\quad and\quad }
H_*(\mbb{RP}^\infty;\mbb{Z}_2)
=\begin{cases}
\mbb{Z}_2 & * \geq 0 \\
0 & \text{otherwise,}
\end{cases}
\]
respectively. This means that $\chi(\mbb{RP}^\infty;\mbb{Z})=1$, but $\chi(\mbb{RP}^\infty;\mbb{Z}_2)$ is undefined.
\end{remark}

\subsection{Isometry}
Before we proceed to the following section, we discuss an essential and straightforward concept called isometry.
This idea is significant and necessary for our study, so it is crucial to cover it before moving on.
First, we present a definition of the isometry:
\begin{definition}
    If a function $f:\mbb{R}^n\rightarrow\mbb{R}^n$ satisfies
    $$
    \|f(x)-f(y)\|=\|x - y\|
    $$
    for every $x, y\in \mbb{R}^n$, then we call $f$ an \textbf{isometry} or \textbf{rigid motion}.
\end{definition}

There are typical examples of isometries in three-dimensional Euclidean space: translations, rotations, and reflections.
Rotations and reflections can be represented by elements of $\on{O}(3)$, the set of orthogonal matrices.
It is also worth noting that generally, in a Euclidean space of dimension $n$, translations and linear functions induced from elements of $\on{O}(n)$ serve as clear examples of isometries.
On the other hand, any isometry can be uniquely expressed as a combination of these functions:
\begin{proposition}\label{prop: isometry-unique}
    If a function $f:\R^n\rightarrow\R^n$ is an isometry, then there are unique $A\in \on{O}(n)$ and $b\in \R^n$ satisfying
    $$f(x) = Ax + b$$
    for every $x\in \R^n$.
\end{proposition}
We omit the proof. 
From Proposition \ref{prop: isometry-unique}, to show that an operator is invariant under isometry on $\R^n$, the above proposition states that it is sufficient to demonstrate that the operator is both $\on{O}(n)$-invariant and translation-invariant.
This greatly simplifies the task of proving invariance under isometry, as it reduces the problem to showing invariance under just two simpler types of transformations.

\section{Theoretical Results and Methods}
\label{sec: method}
In this section, we present our methodology and the corresponding theoretical research.
We introduce the Euler curve transform and the persistent homology transform, and $\on{O}(3)$-equivariant graph neural networks.
Then, Combined together, these transformations provide both isometry-invariant and subdivision-invariant representations of embedded simplicial complexes.
After explaining the theory behind these transformations, we finally present our implementation for taking advantage of these new methods in practice.

\subsection{Euler Curve Transform}

In this subsection, we introduce the Euler curve transform, along with a relevant theorem that characterizes this transformation.
The Euler curve transform is an operator denoted as $\mcal{R}$ that maps from $\on{CF}(\R^n)$ to $\on{CF}(S^{n-1}\times \R)$.
Specifically, the Euler curve transform can be defined through the Euler integral:
$$\mathcal{R}(f)(v, r)=\int_{\mathbb{R}^n} f(x) \cdot 1_{x \cdot v \leq r}(x) \;d\chi(x),$$
where $f$ is a bounded compactly supported constructible function and $1_A$ is the characteristic function of a set $A$.

Now, let us consider a case where $f$ is the characteristic function of an embedded simplicial complex $K$, denoted as $1_K$.
Since $K$ is a semialgebraic set, we know that $1_K$ is a bounded compactly supported constructible function.
In this case, if we apply the Euler curve transform to $1_K$, we obtain:
$$
\mcal{R}(1_K)(v, r) = \int_{\mbb{R}^n} 1_{\{
x\in K \mid x\cdot v \leq r
\}}\; d\chi = \chi(K_{v, r}).
$$
In other words, the Euler curve transform of the characteristic function of an embedded simplicial complex $K$ is simply the Euler characteristic of the intersection of $K$ with a half-space defined by a hyperplane perpendicular to the direction $v$.

A relevant property of the Euler curve transform is that it is injective, as stated in \citet{ghrist2018persistent}:
\begin{theorem}
    The Euler curve transform $\mcal{R}$ is injective.
\end{theorem}
In other words, if there exist two distinct embedded simplicial complexes $K_1$ and $K_2$, the resulting Euler curve transforms $\mcal{R}(1_{K_1})$ and $\mcal{R}(1_{K_2})$ are different.
This property can be used to classify embedded simplicial complexes.
For instance, we can classify images by converting them to embedded simplicial complexes and applying the Euler curve transform as described in \citet{jiang2020weighted}.

To simplify notation, for each embedded simplicial complex $K$, we define a function $\mcal{F}_K$:
\begin{align*}
\mcal{F}_K : S^{n-1}\times \mbb{R}&\longrightarrow \mbb{R}\\
(v, r)&\longmapsto \chi(K_{v, r}).
\end{align*}
Then, by the injectivity, $K_1 \neq K_2$ implies $\mcal{F}_{K_1} \neq \mcal{F}_{K_2}$.

The following proposition shows how the result of the Euler curve transform changes for $\on{O}(n)$-transformations and parallel transformations.
\begin{proposition}\label{prop: ECT equiv}
Let $K$ be an embedded simplicial complex in $\R^n$.
Then, for $R\in \on{O}(n)$ and $w\in \R^n$,
$$
\mcal{F}_{RK + w}(v, r) = \mcal{F}_{K}(Rv, r-v\cdot w).
$$
\end{proposition}

That is, the transform is $\on{O}(n)$-equivariant, and if the embedded simplicial complex is translated, then the resulting function is also translated.

Note that a characteristic function of an embedded simplicial complex does not dependent on the division of the complex:
let $K$ be a simplicial complex and $K^\prime$ be a subdivision of $K$. 
Since the characteristic functions $1_K$ and $1_{K^\prime}$ are the same, the Euler curve transform is subdivision-invariant.

\subsection{Persistent Homology Transform}
\label{subsec: PHT}
Similarly, for an embedded simplicial complex $K$ in $\R^n$, we can consider the persistent homology transform $\mcal{G}_K$ that is defined via the formula
\begin{align*}
\mcal{G}_K : S^{n-1}&\longrightarrow \mcal{B}\\
v&\longmapsto B_K^v
\end{align*}
where $B_K^v$ is the persistence diagram induced by $v$ and $\mcal{B}$ is the set of persistence diagrams.
It can be proved that this transform is stable for a change in a direction of a unit vector:
\begin{proposition}\label{prop: Wasserstein stable}
    Let $K$ be an embedded simplicial complex. 
    With regard to the bottleneck distance, the function $\mcal{G}_K$ is Lipschitz.
\end{proposition}
This indicates that persistence diagrams located in nearby regions above $S^{n-1}$ contain similar information.
To put it differently, knowing the persistent homology information of sufficient directions is practically the same as knowing those of all points.

The persistent homology transform, as the Euler curve transforms, possesses some important properties that are essential for effective data analysis and interpretation:
\begin{proposition}\label{prop: diagram o3 and trans}
    Let $K$ be an embedded simplicial complex in $\R^n$ and $v$ be a unit vector.
    Then, for $R\in\on{O}(n)$, 
    $$
    \mcal{G}_{RK}(v) = \mcal{G}_K(Rv).
    $$
    Furthermore, if $\mcal{G}_K(v) = \{(b_i, d_i)\}_{i\in I}$, we have
    $$
    \mcal{G}_{K+w}(v) = \{(b_i + v\cdot w, d_i + v\cdot w)\}_{i\in I}
    $$
    for $w\in \mbb{R}^n$.
\end{proposition}
As in the case of the Euler curve transform, the operator $\mcal{G}_*$ exhibits $\on{O}(n)$-equivariance, and the entries of the persistence diagram are translated when the embedded simplicial complex is translated.

There is a similar result on the persistence landscape.
\begin{proposition}\label{prop: landscape o3 and trans}
    Let $K$ be an embedded simplicial complex in $\R^n$ and $v$ be a unit vector.
    Then, for $R\in\on{O}(n)$, we have
    $$\lambda_{v, c}^{RK} = \lambda_{Rv, c}^{K}$$
    for every $c\in\N$.

    Furthermore, for $w\in \R^n$, we have
    $$
    \lambda_{v, c}^{K+w}(t) = \lambda_{v,c}^K(t-v\cdot w)
    $$
    for every $c\in \N$.
\end{proposition}

Note that in some cases, the persistence diagram may contain entries with an infinite lifetime.
This occurs when a topological feature, such as a connected component or a hole, exists throughout the entire parameter range being considered.
In the case of Vietoris-Rips complexes, which are commonly used in topological data analysis literature, there is only one entry with infinite persistence in the zeroth homology, corresponding to the overall connectedness of the complex.
However, in our approach, we are not dealing with Vietoris-Rips complexes, which can lead to multiple entries with infinite persistence in the persistence diagram.
These features can theoretically and computationally complicate the analysis, so it is desirable to eliminate them.
To remove the entries with infinite persistence and preserve the information, we can simply use a bijective function from $\overline{\R}$ to a finite interval.
For example, for a persistence diagram $\{(b_i, d_i)\}_{i\in I}$, we can generate a new persistence diagram $\{(b_i^\prime, d_i^\prime)\}_{i\in I}$ such that $b_i^\prime = h(b_i)$ and $d_i^\prime = h(d_i)$ where
$$
h(x)= \begin{cases}a_0 & \text { if } x = \infty, \\ -a_0 & \text { if } x = -\infty, \\ a_0\tanh(a_1x) & \text { otherwise }\end{cases}
$$
for scaling constants $0< a_0, a_1<\infty$. 
Since the function is strictly increasing and the range of the function is finite, we have $b_i^\prime \leq d_i^\prime$ and the corresponding persistence landscape is also well defined via Equation \eqref{eq: pl_different rep}.
For the convenience of notation, we denote the transform from $\{(b_i, d_i)\}_{i\in I}$ to $\{(b_i^\prime, d_i^\prime)\}_{i\in I}$ as $\mcal{H}$, that is, $\mcal{H}(\{(b_i, d_i)\}_{i\in I}) = \{(b_i^\prime, d_i^\prime)\}_{i\in I}$.

As stated in \citet{turner2014persistent} and \citet{ghrist2018persistent}, the persistent homology transform is a sufficient statistic.
Since we can recover the original persistent diagram $\mcal{G}_K(v)$ from $\mcal{H}(\mcal{G}_K(v))$ for an embedded simplicial complex $K$ and for each $v\in S^{n-1}$, $\mcal{H}\circ\mcal{G}_*$ is also a sufficient statistic.
Note that there are only finitely many simplices in an embedded simplicial complex $K$, and therefore the number of entries in $\mcal{H}(\mcal{G}_K(v))$ is finite for every $v\in S^{n-1}$.
Therefore, if we denote the persistence landscape corresponding $\mcal{H}(\mcal{G}_K(v))$ via Equation \eqref{eq: pl_different rep} as $\{\lambda_{v, c}^{\prime K}\}_{c\in\N}$,
the function
\begin{align*}
    S^{n-1}&\longrightarrow \mcal{S}\\
    v&\longmapsto \{\lambda_{v, c}^{\prime K}\}_{c\in\N}
\end{align*}
is also a sufficient statistic where $\mcal{S}$ is the set of persistence landscapes since the conversion from persistence diagrams to persistence landscapes is injective.

Although the persistent homology transform is a theoretically elegant tool for analyzing topological features, it is not the primary focus of our paper.
Instead, we primarily use the Euler curve transform to analyze the topological properties of simplicial complexes in an isometry- and subdivision-invariant manner.
While we acknowledge the potential limitations of the persistent homology transform, we will briefly discuss its use and potential drawbacks in Section \ref{sec: discussion} of our paper.

\subsection{O(3)-Equivariant Graph Neural Networks}
\label{subsec: o3 EQUIV GNN}
A variety of methods have been proposed for constructing operators that are either rotation-invariant or rotation-equivariant for functions defined on the sphere $S^2$.
Some of these methods, such as those introduced in \citet{cohen2018spherical, esteves2018learning}, use correlation and convolution operators on $S^2$ or the rotation group $\on{SO}(3)$, achieved through the use of spherical Fourier transform.
Other approaches, such as those proposed in \citet{defferrarddeepsphere, de2021gauge, yang2020rotation}, are based on graph neural networks.
Furthermore, in \citet{cohen2019gauge}, there was a suggestion to use an isotropic filter to create rotation-equivariant convolutional neural networks.
In this subsection, we present our approach, which is a $\on{O}(3)$-equivariant graph neural network.
Additionally, we provide a theoretical basis for our approach. 
Before proceeding, we introduce a notation that will be used throughout this paper.

Suppose $f$ is a function from the unit sphere $S^{n-1}$, and $R$ is an orthogonal matrix in $\on{O}(n)$. Then, the function obtained by applying the matrix $R$ to the input of $f$, that is, the function $x \mapsto f(Rx)$, is denoted as $Rf$.
This notation will be frequently used to describe equivariant or invariant properties of various functions and operators.

Now, we explain our graph neural networks.
Let us suppose that we are given a function $f$ that maps from the 2-dimensional sphere $S^2$ to a $k$-dimensional vector space over the real numbers.
To construct graph neural networks from this function, we start by uniformly sampling a set of $n$ points $\{x_1, x_2, \dots, x_n\}$ on $S^2$.
Next, we connect pairs of points with an edge if their Euclidean distance is less than a predetermined threshold.
This results in a graph where each node is associated with a feature vector that corresponds to the value of the function $f$ at the corresponding point on the sphere.
With this graph in hand, we can use graph convolutional networks to process the data and extract useful information from the underlying function.
As a basic graph neural network architecture, we can implement a graph convolutional network as follows:
\begin{equation}\label{eq: GCN basic}
\mcal{T}(f)(v) = f(v)\cdot W_0 + \sum_{u\in \mcal{N}(v)}f(u)\cdot W_1
\end{equation}
where $\mcal{T}(f)$ is the output of the GCN layer, $W_0, W_1$ are parameter matrices, and $\mcal{N}$ denotes the neighborhood of $v$.

More generally, we can consider Message Passing Neural Network (MPNN).
MPNN is a type of graph neural network that operates on graph-structured data \citep{gilmer2017neural}.
MPNN uses a message-passing mechanism to update the node representations based on the representations of their neighboring nodes and edge features.
In each iteration of the message passing, the node representations are combined with the representations of their neighbors to produce new representations that capture the information in the graph.
The formulation of MPNN is as follows:
assume that we have an undirected graph $G$ with node features $f(x_i)$ on the node $x_i$ and edge features $e_{ij}$ between $x_i$ and $x_j$.
Then, the new features $\mcal{T}(f)(x_i)$ are calculated by 
\begin{equation}\label{eq: MPNN}
\mcal{T}(f)(x_i) = U\left(f(x_i), \sum_{x_j\in \mcal{N}(x_i)} M(f(x_i), f(x_j), e_{ij})\right)
\end{equation}
where $U$ is a vertex update function and $M$ is a message function.

For our $\on{O}(3)$-equivariance architecture, let us consider the complete graph constructed from $\{x_1, x_2, \dots, x_n\}$ and set the edge features $e_{ij}$ to the distance between the nodes $x_i$ and $x_j$.
Also, we define the functions $M(x, y, z) := g_\theta(z)\cdot y $ and $U(x, y) := \frac{1}{n}(g_\theta(0)\cdot x + y)$, where $g_\theta$ is a function that can be trained, mapping from $\R$ to the space of $k\times \ell$ matrices. 
Then, Equation \eqref{eq: MPNN} become
\begin{equation}\label{eq: MPNN ours}
\mcal{T}(f)(x_i) = 
\frac{1}{n}\left(
    g_\theta(0)\cdot f(x_i)  + \sum_{x_j\in \mcal{N}(x_i)}g_\theta(e_{ij})\cdot f(x_j)
    \right) = \frac{1}{n}\sum_j g_\theta (e_{ij})f(x_j).
\end{equation}

Also, if we define a $n\times n$ matrix $(G)_{ij} := \frac{1}{n}g_\theta(e_{ij})$, Equation \eqref{eq: MPNN ours} can be represented by 
$$
\mcal{T}(f) = Gf.
$$
We present Equation \eqref{eq: MPNN ours} as our $\on{O}(3)$-equivariant operator, and we provide theoretical justification for its $\on{O}(3)$-equivalence as follows.

Let $g$ be a function from $S^2$ such that $g(x) = g_\theta(\|x\|)$.
Using $g$, we can extend the function $\mcal{T}(f)$ to the domain $S^2$:
$$
\mcal{T}(f)^\prime(x) = \frac{1}{n}\sum_j g(x - x_j)f(x_j).
$$
Of course, we have $\mcal{T}(f)(x_i) = \mcal{T}(f)^\prime(x_i)$ for every $i = 1, 2,\dots, n$.
The following theorem shows a theoretical background of our graph neural network.
\begin{theorem}\label{thm: GNN equiv}
    Let $f$ and $g$ be bounded functions on $S^2$ with respect to $\|\cdot \|_\infty$.
    Assume that $x_1, x_2, \dots, x_n$ are independent identically distributed random variables on $S^2$.
    Then, for $R\in \on{O}(3)$ and $\epsilon > 0$, 
    $$\mbb{P}\left[
        \| R\mcal{T}(f)^\prime(x) - \mcal{T}(Rf)^\prime(x)\|_\infty > \epsilon
    \right] \lrr 0
    $$
    for every $x\in S^2$ as $n$ goes to infinity.
\end{theorem}
Thus, if we construct a graph neural network as described previously by sampling points uniformly over the sphere, we will be approximating an $\on{O}(3)$-equivariant operator.
Moreover, when we increase the number of points we sample, the ``equivariance error'' reduces, leading to a better and more precise representation of the input data.

There are other approaches using graph neural networks to achieve rotational equivariance. In \citet{defferrarddeepsphere} and \citet{khasanova2017graph}, the following convolution is used: for $f:S^2\rightarrow \mbb{R}$,
$$\mcal{T}(f)=\left(\sum_{k=0}^P \alpha_k L^k\right) (f)$$
where $L$ is a weighted Laplacian matrix and the weight of an edge is a function of the length of the edge.
Therefore, $L$ is a fixed specific form of the matrix $G$.
However, as stated in \citet{defferrarddeepsphere}, the equivariance error varies depending on $G$.

Practically, we do not need to construct the complete graph. 
Since Equation \eqref{eq: GCN basic} is a specific form of Equation \eqref{eq: MPNN ours}, we can just use the basic graph convolutional networks:
for a predetermined threshold $r>0$, if we set 
$$
g_\theta(x)= \begin{cases}nW_0 & \text { if } x = 0, \\ nW_1 & \text { if } 0<x<r, \\ 0 & \text { otherwise, }\end{cases}
$$
then we induce Equation \eqref{eq: GCN basic} from Equation \eqref{eq: MPNN ours}.
We believe that the small support of $g_\theta$ corresponds to the small kernel size in the $2$-dimensional convolution operator.

By adding a global max or average pooling layer after our graph neural network, we can obtain a representation of the underlying graph.
It should be noted that such pooling layers are capable of representing an $\on{O}(3)$-invariant operator since the maximum or average does not change under $\on{O}(3)$-transformations.
If we stack our graph neural networks and pooling layer together, according to Theorem \ref{thm: GNN equiv}, the network possesses an $\on{O}(3)$-invariant property.
This means that using our proposed graph neural network and a pooling layer allows us to obtain an $\on{O}(3)$-invariant representation of an embedded simplicial complex.

\subsection{Isometry-Invariant and Subdivision-Invariant Operator}
This section presents the development of isometry-invariant and subdivision-invariant operators by leveraging the Euler curve transform and an $\on{O}(3)$-invariant operator.
A brief overview of their implementation will also be provided.

First, the Euler curve transform can be considered as a function
\begin{alignat*}{2}
    \mcal{F}_K:S^{n-1}&\longrightarrow \on{Map}(\R\rr\R) \\
    v&\longmapsto  \quad(\mcal{F}_K(v):&\mbb{R} &\longrightarrow\mbb{R})\\
    &&r&\longmapsto\chi(K_{v, r})
\end{alignat*}
where $\on{Map}(A\rr B)$ denotes the set of functions from $A$ to $B$.
To make a translation-invariant operator, we use translation-invariant functionals on the set $\on{Map}(\R\rr\R)$:
\begin{proposition}\label{prop: inv and equiv}
Assume that there are translation-invariant functionals $\{\mcal{D}_i\}_{i=1}^m$ on the set $\on{Map}(\R\rr\R)$, that is, if there exist $t\in \mbb{R}$ such that $f(x) = g(x + t)$ for every $x$, then $\mcal{D}_if = \mcal{D}_ig \in \mbb{R}$ for every $1\leq i\leq m$.
Let $\mcal{DF}_K$ be a function 
\begin{alignat*}{2}
\mcal{DF}_K:S^{n-1}&\longrightarrow \mbb{R}^m\\
v&\longmapsto \{\mcal{D}_1\circ\mcal{F}_K(v), \dots, \mcal{D}_m\circ\mcal{F}_K(v)\}.
\end{alignat*}
Then, $\mcal{DF}_*$ is $\on{O}(n)$-equivariant, subdivision-invariant, and translation-invariant on the set of embedded simplicial complexes.
\end{proposition}

There is a similar result for the persistence landscape:
\begin{proposition}\label{prop: inv and EQUIV PL version}
Assume that there are functionals $\{\mcal{E}_i\}_{i=1}^m$ on the set $\on{Map}(\R\times \N \rr\overline{\R})$ such that if there exist $t\in \mbb{R}$ such that $f(x, c) = g(x + t, c)$ for every $x$ and $c$, then $\mcal{E}_if = \mcal{E}_ig \in \mbb{R}$ for every $1\leq i\leq m$.
Let $\mcal{E}\lambda_K$ be a function 
\begin{align*}
\mcal{E}\lambda_K:S^{n-1}&\longrightarrow \mbb{R}^m\\
v&\longmapsto \{\mcal{E}_1\circ\lambda_v^K, \dots, \mcal{E}_m\circ\lambda_v^K\}
\end{align*}
where $\lambda_v^K(x, c) := \lambda_{v,c}^K(x)$.
Then, $\mcal{E}\lambda_*$ is $\on{O}(n)$-equivariant, subdivision-invariant, and also translation-invariant on the set of embedded simplicial complexes.
\end{proposition}

As a direct result of Proposition \ref{prop: inv and equiv}, we present the following corollary:

\begin{corollary}\label{cor: our goal}
    Let $\mcal{P}$ be an $\on{O}(n)$-invariant function
    \begin{align*}
        \mcal{P}: \on{Map}(S^{n-1}\rr \R^m)\lrr \R^k, 
    \end{align*}
    that is, $\mcal{P}(Rf) = \mcal{P}(f)\in \R^k$ for $R\in \on{O}(n)$.
    Then, $\mcal{P}(\mcal{DF}_*)$ is isometry-invariant and subdivision-invariant on the set of embedded simplicial complexes.
\end{corollary}

The above corollary represents the desired operator that we aim to approximate in our analysis.
To achieve this goal, we employ several deep-learning architectures to approximate and implement the various components that make up this corollary.
In other words, we use deep learning techniques to create a model that can effectively approximate an isometry-invariant and subdivision-invariant operator.

To obtain translation-invariant functionals $\{\mcal{D}_i\}$, we propose using $1$-dimensional convolutional neural networks and global pooling.
We use the fact that $1$-dimensional convolutions are known to be translation-equivariant.
After applying a series of convolutions, by adding a global pooling layer at the end of the network, we can obtain a vector for sampled points on $S^2$ that is translation-invariant.
As a global pooling layer, we can use the global max pooling layer or a global average pooling layer.
This pooling layer aggregates the outputs of the convolutional layers across the entire input, resulting in a single output that is invariant to translations.
The resulting vector can be used as a representation of the Euler curve for each sampled point and can be fed into a graph neural network for further processing.
After that, we apply graph neural networks.
Section \ref{subsec: o3 EQUIV GNN} of the paper outlines how we can use our proposed graph neural network in combination with a pooling layer, such as the global max pooling or the global average pooling layer, to effectively approximate an $\on{O}(3)$-invariant operator.
We discuss the implementation in a little more detail in the following.

\begin{figure}[t]
    \centering
    \includegraphics[width=0.9\linewidth]{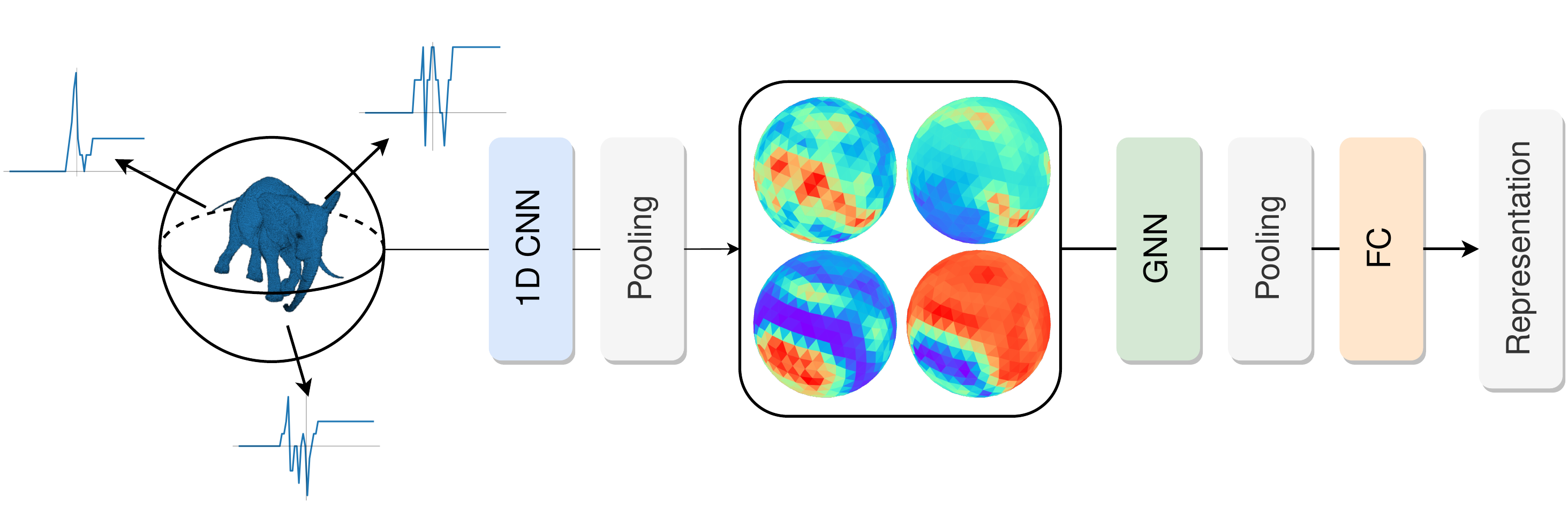}
    \caption{Illustration of a simplified structure of the proposed architecture.
    Initially, the Euler curves are obtained by pre-processing an embedded simplicial complex.
    Next, a $k$-dimensional vector is produced for sampled points on $S^2$ by using a 1D convolutional neural network and a pooling layer.
    Subsequently, we use our graph neural network over the graph of the points on the sphere.
    Finally, a representation vector of the embedded simplicial complex is obtained via a pooling layer.
    To illustrate this, we use a subdivision of the icosahedron and a mesh from the ANIM data set in this figure.
    }
    \label{fig: architecture}
\end{figure}

\subsubsection{Implementation}
Now, we explain our procedure in order briefly.
To facilitate the discussion, we illustrate a simplified structure of our architecture in Figure \ref{fig: architecture}.
To begin our procedure, we first need to sample points on $S^2$ as uniformly as possible.
A variety of methods are available to sample points fairly uniformly on $S^2$, such as using subdivisions of the icosahedron, HEALPix sampling \citep{gorski2005healpix}, Sukhrev grid sampling \citep{yershova2004deterministic}, or Fibonacci grid sampling \citep{swinbank2006fibonacci}.
The resulting set of points denoted as $\{x_1, x_2,\dots, x_n\}$, are used to create a graph $G$, which is used for the subsequent steps in our approach.

Next, we make a graph $G$ based on the sampled points.
As we discussed in Section \ref{subsec: o3 EQUIV GNN} of our paper, there are two ways to generate the graph: creating a complete graph or connecting pairs of points that are closer than a certain distance.

After creating the graph, we need to preprocess all the embedded simplicial complex data.
Specifically, we generate a discretized Euler curve for the points we previously sampled.
The discretized Euler curve is represented as a vector in $\R^t$, where $t$ is a hyperparameter that determines the resolution of the curve.
We note that this transformation produces the same curve, regardless of how an embedded simplicial complex is divided.

Once we have generated the discretized Euler curve for the sampled points, we apply multiple 1D convolutional neural networks and apply a pooling layer at the end.
This generates a translation-invariant vector for each point $x_i$.
We use these vectors as node features for the input of the next step.

Using the node features on graph $G$, we stack several of our proposed graph neural networks on it to generate a new feature for each node in the graph.
By doing this, we can represent the embedded simplicial complex more effectively.
Next, we obtain a graph representation using a global max pooling layer or a global average pooling layer.
These pooling layers aggregate the node features and produce a single vector representation for the entire graph.
This graph representation approximates an isometry-invariant and subdivision-invariant representation of the embedded simplicial complex.
This representation vector can then be used for further analysis and applications.

In Section \ref{sec: exp}, we provide a detailed description of our experiments. It includes our methodology, our data set, results, and hyperparameters.

\subsection{SO(3)-Invariant and not O(3)-Invariant Operator}
This subsection considers the topic of an operator that is $\on{SO}(3)$-invariant yet is not $\on{O}(3)$-invariant, indicating the operator's capability to distinguish inversion.
Although a significant number of studies have been conducted on $\on{SO}(3)$-invariant operators, only a few have discussed the operators that can differentiate reflection.
Here, we present an example of a particular $\on{SO}(3)$-invariant operator that can distinguish reflection.

Let $f$ and $g$ be functions from $S^2$ to $\R$ where $g$ is a learnable function.
Then, from \citet{cohen2018spherical}, define a $\on{O}(3)$-equivariant operator $f\star g$:
\begin{align*}
    f\star g:\on{SO}(3)&\longrightarrow \R\\
    x&\longmapsto \int_{S^2}g(x^{-1}y)f(y)\; dV(y).
\end{align*}
As a simple example of $\on{SO}(3)$-invariant operator, we present 
$$f\bullet g := \max\{(f\star g)(x) \mid x\in \on{SO}(3)\}$$
which can distinguish between a function and its reflection:
let $A$ be an open set of the shape ``L'' on $S^2$, and $f$ be the characteristic function of $A$.
Also, assume that $B$ is a reflection of $A$ on $S^2$ and $g$ is its characteristic function.
See Figure \ref{fig: L_sphere} for the illustration.
\begin{figure}[t]
\centering
\begin{subfigure}{.25\textwidth}
    \centering
    \includegraphics[width=\linewidth]{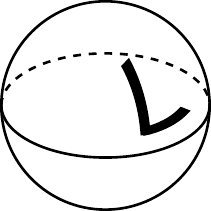}
    \caption{function $f:S^2\rightarrow \mbb{R}$}
\end{subfigure}
\hspace{1.0in}
\begin{subfigure}{.25\textwidth}
    \centering
    \includegraphics[width=\linewidth]{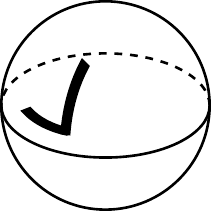}
    \caption{function $g:S^2\rightarrow \mbb{R}$}
\end{subfigure}
\caption{Examples of two functions from $S^2$ to $\mbb{R}$. The function $g$ is a reflection of the function $f$.}
\label{fig: L_sphere}
\end{figure}
Then, we have
$$f\bullet f= \operatorname{Vol}(A) = \operatorname{Vol}(B) = g\bullet g$$
and 
$$f\bullet g= \max\{\operatorname{Vol}(A\cap RB)\mid R\in \on{SO}(3)\} \neq \on{Vol}(A)$$
since the shape of $A$ and $B$ are not overlapped completely just with rotations.

Therefore, operator $f\bullet g$ is $\on{O}(3)$-invariant and capable of distinguishing reflection, which can be useful for certain applications.
This operator can be implemented by finding the function $f\star g$ through the use of the spherical Fourier transform \citep{cohen2018spherical} and selecting the maximum value over the points sampled in $\on{SO}(3)$.
In our proposed operator in Section \ref{subsec: o3 EQUIV GNN}, if we replace the final graph pooling layer with the pooling step $f\bullet g$, we can create an operator that is both invariant to $\on{SE}(3)$ and able to discriminate reflection, resulting in a more expressive model.

\section{Experiments}
\label{sec: exp}
This section covers our experimental results showing the outcomes of our experiments with synthetic mesh data.

\subsection{Data Set}
First, we describe our experiments on the ANimals in Motion (ANIM) data set.
The ANIM data set is a collection of synthetic mesh sequences obtained from \citet{sumner2004deformation}, containing eight different categories: elephant, camel, horse, flamingo, cat, face, head, and lion.
Each category contains a varying number of mesh sequences.
In our experiment, we aim to test the performance of our model in obtaining the representation of mesh data.
To show the representations, our model is trained to gather the mesh data in each class around the vertex of the regular octagon in $S^1\subset\mbb{R}^2$ (see Figure \ref{fig: ANIM high res}), and we tested its ability to be isometry-invariant through a series of transformations.

We have different numbers of objects for each class, like elephants, camels, horses, flamingos, cats, faces, heads, and lions.
For example, there are $60$ elephants, $59$ camels, and $59$ horses. 
The number of objects in the other classes is smaller, with only $10$ cats, faces, heads, and lions, and $11$ flamingos.
We also limit the number of training data to $5$ randomly chosen mesh sequences per class, resulting in a total of only $40$ meshes out of the $229$ meshes without any data augmentation
This small training set size allows us to evaluate the model's ability to learn from limited data and generalize well to new examples.

\subsection{Model Architecture}
\begin{figure}[t]
    \centering
    \includegraphics[width=0.6\linewidth]{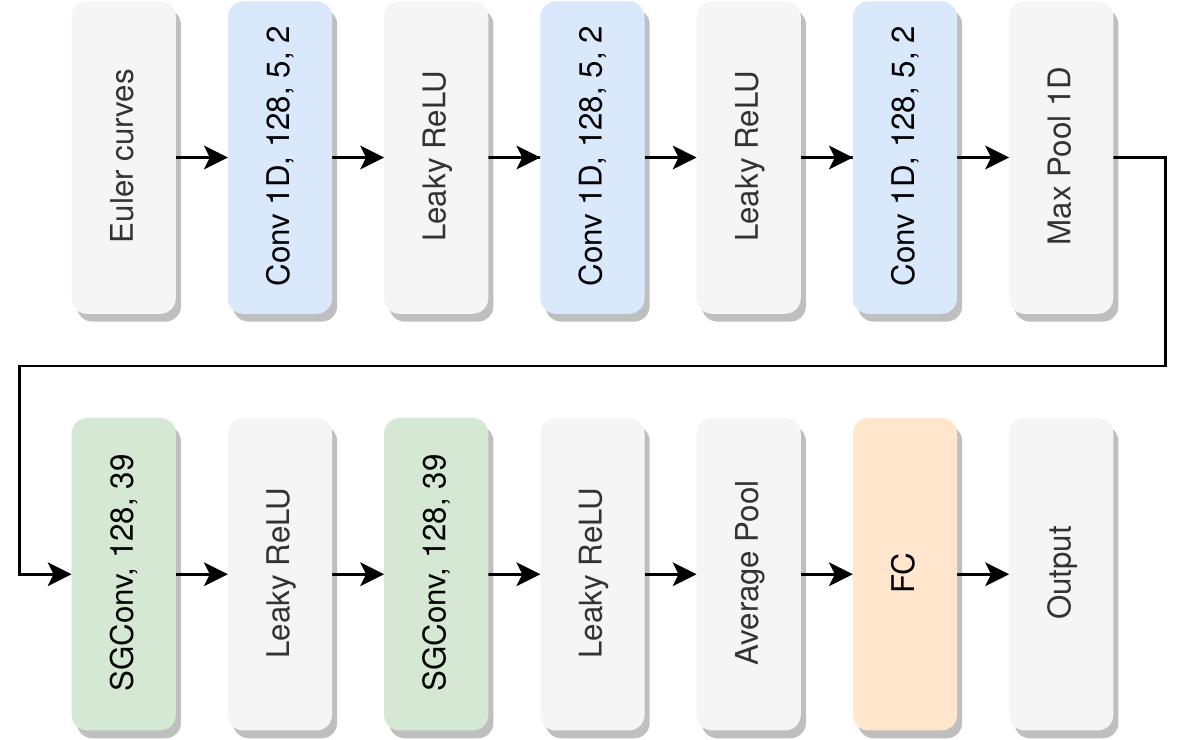}
    \caption{
        This figure shows the architecture that is used in Section \ref{sec: exp}.
        The convolutional layers are labeled Conv 1D, followed by the number of channels, kernel size, and then stride.
        For the Leaky ReLU layer, we fix the negative slope as $0.01$.
        The SGConv layers are labeled SGConv, followed by the number of channels and $k$, the hyperparameter of SGConv.
        }
    \label{fig: model detail}
\end{figure}
In this subsection, we present the model architecture that we employed for our experiment.
The initial step is to use the icosahedron's subdivision as described in \citet{cohen2019gauge,defferrarddeepsphere,de2021gauge,yang2020rotation}.
We use the term ``$\text{level}$'' of subdivision to present how subdivision was created.
For instance, $\text{level}=1$ corresponds to the icosahedron with 12 vertices, while $\text{level}=i$ represents the subdivision of the edges of the icosahedron triangle into $i$ edges.
We also provide the number of vertices corresponding to each level to avoid ambiguity.
In addition, to make a graph, we take the edges of subdivisions of the icosahedron.

We obtain translation-invariant vectors for each node by using three 1D convolutional neural networks and two Leaky ReLU layers between them and applying a maximum pooling layer at the end.
The first convolution layer changes the number of channels from $1$ to $128$, followed by the second and third convolution layers, which fix the number of channels to $128$.
We use a kernel size of $5$ and a stride of $2$ for all convolution layers.
In the end, we add the max pooling layer.

For a graph neural network, we use SGConv-Leaky ReLU-SGConv-Leaky ReLU architecture where SGConv denotes Simplifying GCN \citep{wu2019simplifying}:
let $D$ be a diagonal matrix where $D_{v, v}=\operatorname{Deg}(v)=|N(v)|$ and define $\widetilde{A}=D^{-1 / 2} A D^{-1 / 2}$ where $A$ is the adjacency matrix of the graph.
For the matrix form of all node features $X$ and a learnable matrix $W$, the transformation of Simplifying GCN is represented by the formula $\widetilde{A}^k X W$.
For our experiment, we use $k=39$ to make sure a receptive field is wide enough.
It should be noted that SGConv can be understood as a series of graph convolutions that do not involve non-linear activation functions.
We fix the number of channels to $128$.
After the last Leaky ReLU layer, we add the average pooling layer so that we have a representation of the graph.
In the end, using a fully connected layer, we decrease the dimension to $2$ to show the result.

In order to enhance the clarity of our model, we introduce Figure \ref{fig: model detail}, which is a graphic depiction of the architecture used in our experiments.

\subsection{Result}
To test the isometry-invariance of our model, we performed a series of random parallel movements, reflections, and rotations on all of the data sets, including those in the training set, and tested the model's ability to see the invariance.
Our results are presented in Figure \ref{fig: ANIM high res}, where we show the performance and invariance of our model under various transformations.

In our experiments, we observed that the mesh data of each class tend to gather together in $\R^2$ to form a well-defined cluster, indicating that the model is able to effectively learn the representations even though the number of the training set was very small.
This clustering behavior is particularly remarkable given that we did not use any data augmentation techniques to increase the size of our training set and to be robust to $\on{O}(3)$-transformation.
Furthermore, we observe that the output obtained from performing three independent random transformations on the data set exhibits a high degree of similarity.
In other words, the resulting representations are very consistent across $\on{O}(3)$-transformations.

\begin{figure}[t]
\centering
\begin{subfigure}{.32\textwidth}
    \centering
    \includegraphics[width=\linewidth]{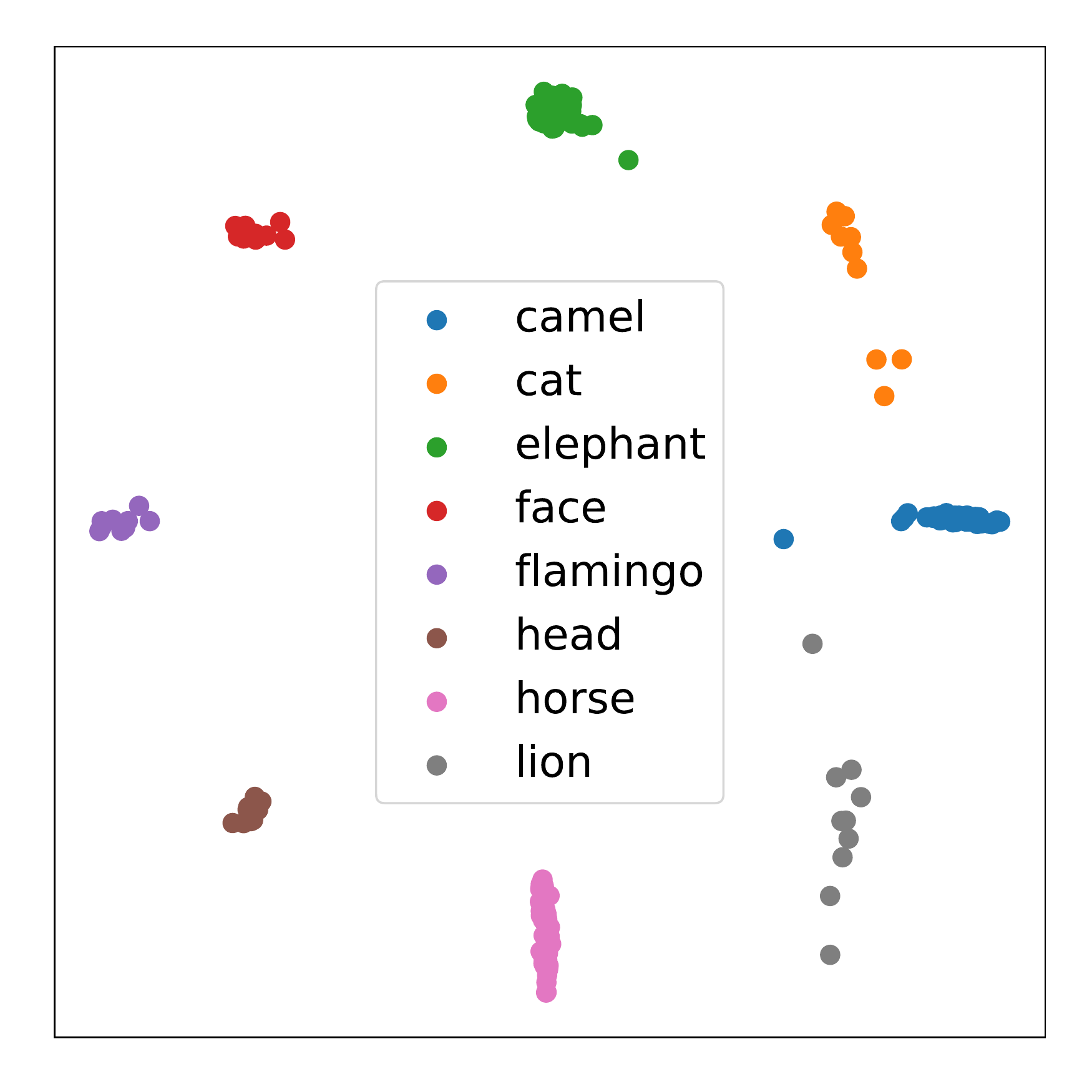}
    \caption{}
\end{subfigure}
\begin{subfigure}{.32\textwidth}
    \centering
    \includegraphics[width=\linewidth]{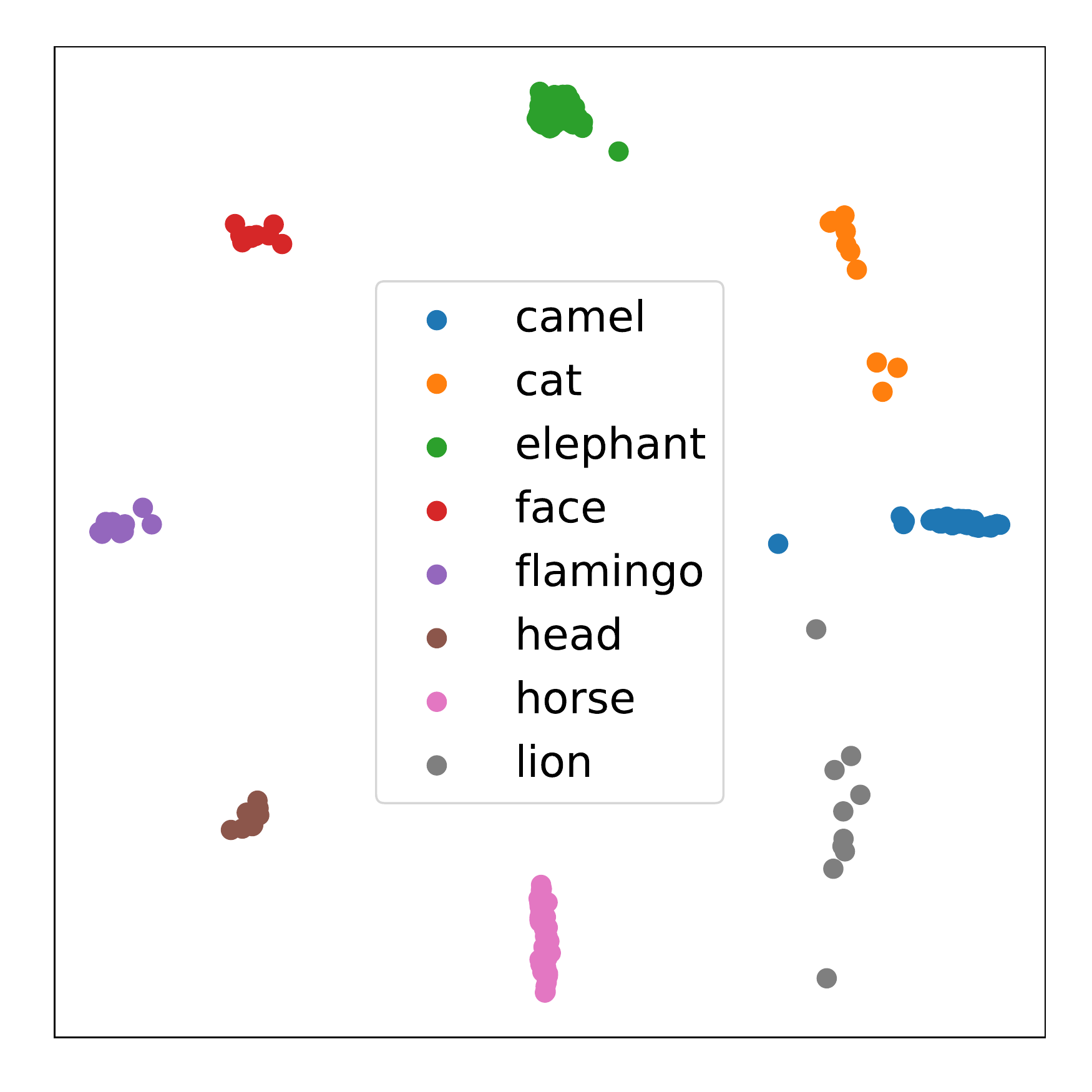}
    \caption{}
\end{subfigure}
\begin{subfigure}{.32\textwidth}
    \centering
    \includegraphics[width=\linewidth]{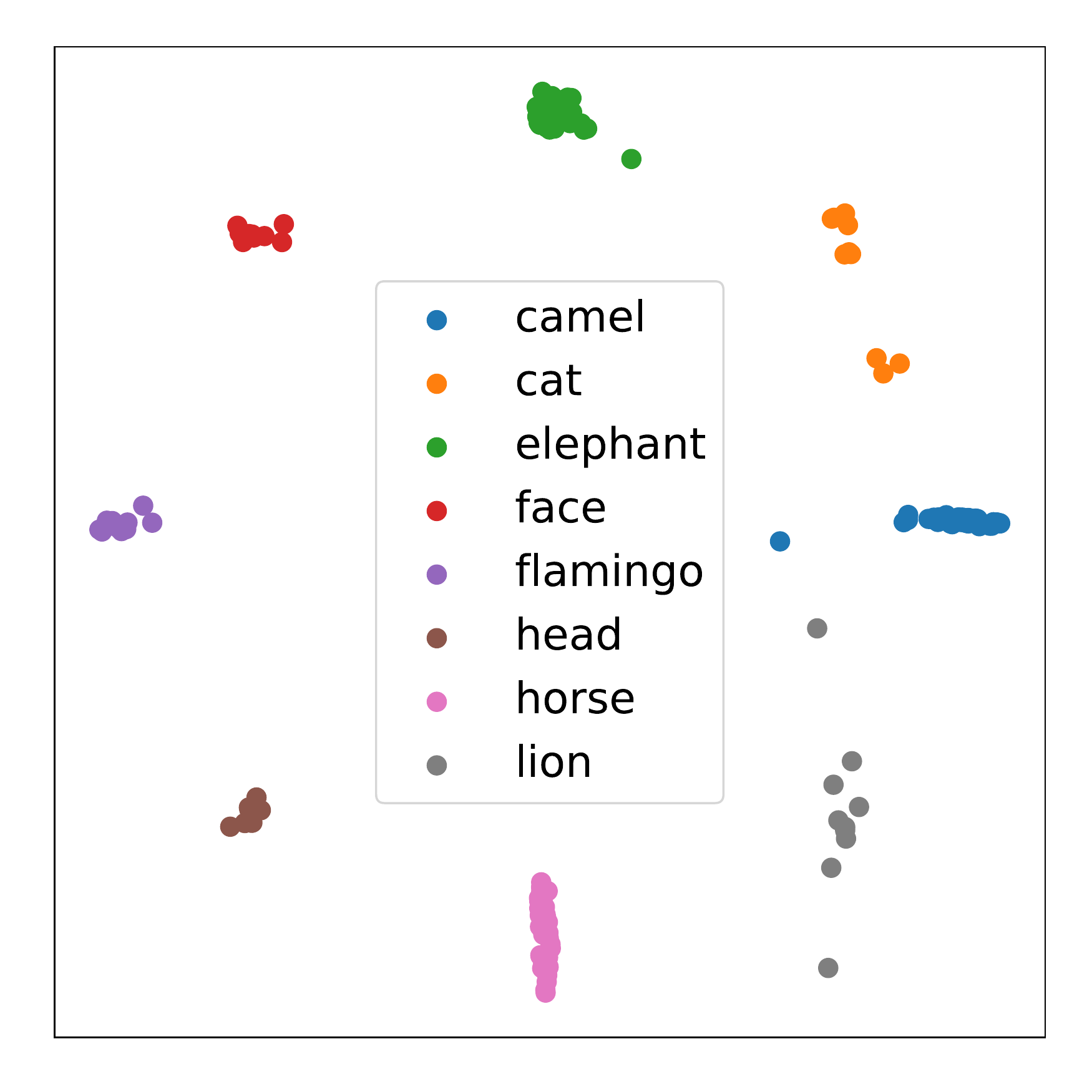}
    \caption{}
\end{subfigure}
\caption{The outcome of our model that is trained with $5$ mesh data per class. 
Each subfigure is produced after we apply random transformations to the data set.}
\label{fig: ANIM high res}
\end{figure}

\subsection{Experimental Details}
In this subsection, we provide a more detailed explanation of the methodology we use for our study.
We begin by downloading all the ANIM data and use $229$ meshes among $337$ data, excluding the horse-collapse and camel-collapse folders.
We exclude these folders because the ``collapsing'' mesh data is no longer geometrically in the same class as horse or camel\footnote{For the detailed description of the data set, refer to \url{http://people.csail.mit.edu/sumner/research/deftransfer/data.html}.}.
Since there are some duplicate vertices and faces, we eliminate them.

To sample points on $S^2$, we first subdivide the icosahedron and take its vertices.
For each sampled point, we compute a discretized Euler curve for all meshes.
To calculate the Euler curve by discretization, we first need to determine the range and resolution of the Euler curve.
We take the Euler curve from $-a$ to $a$ for a positive number $a$. 
Given a resolution value $t$, we divide the interval $(-a, a)$ into $t$ regular pieces to obtain $-a = x_0 <x_1< \dots < x_{t-1} < x_t = a$.
For each $v\in S^2$, we computed $\chi(K_{v, x_i})$ to obtain the discretized Euler curve for $i=1, 2,\dots t$.
In our experiments, the value of $a$ used is $8$.
To ensure that all the meshes fall within the range of $(-a, a)$, we divide the coordinates of the vertices of each mesh by the standard deviation of the vertex coordinates.

In Figure \ref{fig: ANIM high res} of the previous experiment, a resolution of $512$ is used for the Euler curve, along with a subdivision of the icosahedron at level $13$, resulting in $1692$ vertices.

In our experiments, we apply randomized $\on{O}(3)$-transforms by first multiplying the vertices of each mesh by a random orthogonal matrix, followed by sampling a point from the three-dimensional standard normal distribution and adding it to the vertices of the mesh.
We use SciPy \citep{2020SciPy-NMeth} library to obtain the random orthogonal matrix.

After preprocessing all mesh data into the Euler curve, we are ready to train our architecture.
We use the PyTorch \citep{paszke2019pytorch} library to create and train our model.
We create a training set using a random selection of $5$ data for each class, totaling $40$ data, out of a total of $229$ data.
No data augmentation is used in this process.
We use $400$ epochs for training, and the batch size is set to $16$.
We use SmoothL1Loss as the loss function with the hyperparameter beta set to $0.1$.
We start with a learning rate of $0.001$ and reduced it to $0.0001$ after $200$ epochs.

\subsection{Isometry-Invariance Error Analysis}
\label{subsec: isometry-invariance error analysis}
\begin{table}[t]
\centering
\begin{tabularx}{\linewidth}{l*{6}{X}}
\toprule
Resolution of the Euler curve     & 32 & 64 & 128 & 256 & 512& 1024 \\ \midrule
Mean of standard deviation         &  0.2362        &   0.1727   &  0.0206  & 0.0081      &    0.0076 &   0.0099    \\ \bottomrule
\end{tabularx}
\caption{Mean of standard deviations according to the resolution of the Euler curve.}
\label{table: res EC}
\end{table}
In this section, we evaluate the isometry invariance of our model.
To perform this evaluation, we first train a model with the same task as in the previous experiment in Figure \ref{fig: ANIM high res}, using only $5$ data points per class.
We then apply $10$ random transformations to each data and calculate the standard deviation of the distance from the average representation to each representation of the transformed data.
We repeat this process $8$ times and calculate the average value of the standard deviation to determine the ``isometry invariance error''.
Smaller values of the isometry-invariance error indicate that the model is more robust to isometry transformations.

Our first set of measurements focuses on the effect of increasing the resolution from $32$ to $1024$ while keeping the subdivision level fixed at $10$. 
We present the results of these experiments in Table \ref{table: res EC}.
An important point to note is that while the error decreases as the resolution increases, it starts to increase again beyond a certain point, instead of continuing to decrease.

In another experiment, we investigated the isometry invariance error by varying the subdivision level while keeping the resolution fixed at $512$.
We examine the average of standard deviations while increasing the subdivision level from $1$ to $13$.
The results are presented in Table \ref{table: res S2}, and it shows that the error consistently decreases as the subdivision level increases.

\begin{table}[t]
\centering
\begin{tabularx}{\linewidth}{l*{5}{X}}
\toprule
Subdivision level(\#points)    & 1(12) & 4(162) & 7(492) & 10(1002) & 13(1692) \\ \midrule
Mean of standard deviation         &  0.1062      & 0.0160   & 0.0078    &  0.0076   & 0.0076  \\ \bottomrule
\end{tabularx}
\caption{Mean of standard deviations according to the level of the subdivision of the icosahedron.}
\label{table: res S2}
\end{table}

\section{Discussion}\label{sec: discussion}
In this paper, we have explored a novel architecture using the Euler curve transform and $\on{O}(3)$-equivariant operator to obtain an isometry-invariant and subdivision-invariant representation of an embedded simplicial complex.
However, there are limitations to our methodology that need to be addressed.
First, obtaining an Euler curve for each direction can be time-consuming and requires a significant amount of memory, particularly as the resolution and the number of points drawn from $S^2$ increase.
Nonetheless, this process is only required once during the data preprocessing step.

Second, in this paper we are assuming that any subdivision of an embedded simplicial complex preserves the shape of the original complex, resulting in subdivision invariance.
However, in some cases, particularly when dealing with mesh data, subdivisions can change the shape of the complex, leading to a smoothed-out mesh.
In these cases, the characteristic function of the complex changes, and the guarantee of subdivision invariance cannot be maintained.

Third, the quality of the mesh data can significantly impact the topological information, such as the Euler characteristic, and outliers can distort the results.
That is, if the mesh data contains small holes or anomalies, even if they may be difficult to spot with the naked eye, they can have a considerable impact on the topological information obtained from the data. 

As a future direction, incorporating information on persistent homology may help address issues with outliers and small holes. 
Since the Euler curve can be calculated from persistent homology, the persistence barcode or persistence landscape provides more information than the Euler curve.
Also, since it has bottleneck stability \citep{oudot2017persistence} and stability in the direction of the unit vector (Proposition \ref{prop: Wasserstein stable}), it is a more comprehensive tool for handling issues such as small holes and outliers. 
To vectorize persistent homology information, we suggest considering using persistence landscape \citep{bubenik2015statistical} or persistence image \citep{adams2017persistence}.

However, there is an issue with entries that have infinite persistence in the persistence diagram.
Entries with infinite persistence in their case correspond to those whose points are at $\R\times \infty$.
This poses a challenge for using these diagrams as input to deep learning architectures, as incorporating points at infinity requires a new discovery or invention.
The persistence landscape also faces a similar issue due to the potentially infinite values of the landscape function.
Therefore, considering Proposition \ref{prop: inv and EQUIV PL version}, we must develop functionals capable of using all the information in $\lambda_v^K$, which may contain infinite values.
To address this, we proposed a transformation of the diagram in Section \ref{subsec: PHT} to make both the persistence diagram and the persistence landscape easier to handle.
This seems to solve the problem of infinite values since there is no infinite persistence, but it does come with a downside: the transformation eliminates translation invariance.
In our method, we eliminated translation dependency through 1D convolution and pooling layers.
However, when using the transformed landscape or diagram, translation dependency cannot be eliminated in the same way.
This is because the transformed diagrams no longer have the same underlying structure as the original diagrams.

In that respect, we consider using zigzag persistent homology as a potential area of future work.
we believe that there are several benefits to using this approach.
One major advantage is that it allows for the elimination of entries with infinite persistence, while still preserving all of the information.
This is particularly important, as persistent homology transforms currently face an obstacle when dealing with entries that have infinite persistence.
Another key benefit of using zigzag persistent homology is that the zigzag persistent homology data obtained from the filtration induced by the unit vector $v$ and the one induced by $-v$ have exactly the same information \citep{carlsson2009zigzag}.
This offers new opportunities for research and analysis.
The use of zigzag persistent homology could be particularly promising, as it offers a means of addressing the issue of infinite persistence without sacrificing other important properties.

In conclusion, we want to emphasize the potential of our proposed methodology to incorporate weighted Euler curve transforms, as shown in the recent work of \citet{jiang2020weighted}.
This extension will enable us to handle objects with weights associated with each simplex, such as features, which can enhance the versatility of our approach.
By incorporating this extension, we expect to obtain more precise representations of the underlying geometric structures of data and perform better in tasks such as classification and clustering.

\acks{
Taejin Paik was supported by National Research Foundation of Korea (NRF) Grant 2022R1A5A6000840 funded by the Korean Government and also supported by National Science Foundation of Korea Grant funded by the Korean Government (MSIP) [RS-2022-00165404].
The GPUs used in this study were bought with the support of National Research Foundation (NRF) Grant 2019R1A2C4070302 funded by the Korean Government.
}

\clearpage
\appendix
\section{Proofs}\label{appensix: sec: proof}
\subsection{Proof of Theorem \ref{thm: finite field persistent}}
\begin{proof}
    We write $\mcal{C}_i(K; G)$ and $d_i^G$ for the set of $i$-chains and the boundary map respectively with $G$-coefficients.
    Writing out the definitions, and using the dimension theorem gives us
    \begin{align*}
    \sum_{i=0}^n (-1)^i \operatorname{rk}_G H_i(K;G) &= \sum_{i=0}^n (-1)^i (\operatorname{rk}_G \ker d_i^G - \operatorname{rk}_G d_{i+1}^G) \\
    &= \sum_{i=0}^n (-1)^i (\operatorname{rk}_G \ker d_i^G + \operatorname{rk}_G d_i^G)\\
    &= \sum_{i=0}^n (-1)^i \operatorname{rk}_G C_i(K;G) =  \sum_{i=0}^n (-1)^i \operatorname{rk}_G (C_i(K;\mbb{Z}) \otimes G) \\
    &=\sum_{i=0}^n (-1)^i \operatorname{rk}_\mbb{Z} C_i(K;\mbb{Z}) =\sum_{i=0}^n (-1)^i \operatorname{rk}_\mbb{Z} H_i(K;\mbb{Z})\\
    &= \chi(K).
    \end{align*}
\end{proof}

% \subsection{Proof of Proposition \ref{prop: isometry-unique}}
% \begin{proof}
% Since $g(x):= f(x) - f(0)$ is an isometry with $g(0) = 0$, there is an unique matrix $A$ such that
% $$g(x) = Ax$$ 
% satisfying $A^*A$ is the identity matrix from \cite{lax2007linear}.
% Since the function $g$ is a function from $\mbb{R}^n$ to $\mbb{R}^n$, $A$ is a real-valued matrix, and therefore $A\in O(n)$.
% If we set $b:=f(0)$, then we have
% $$
% f(x) = g(x) + f(0) = Ax + b.
% $$
% The uniqueness is obvious from the proof.
% \end{proof}

\subsection{Proof of Proposition \ref{prop: ECT equiv}}
\begin{proof}
    From the equations
    \begin{align*}
    (RK + w)_{v, r} &= \{ x\in \mbb{R}^n \mid x \in RK + w \text{ and } x\cdot v \leq r \}\\
    &= \{ x \in \mbb{R}^n \mid R(x-w) \in K \text{ and } R(x-w)\cdot Rv \leq r- w\cdot v \}\\
    &= \{ R^{-1}x + w \in \mbb{R}^n \mid x \in K \text{ and } x\cdot Rv \leq r- w\cdot v \}\\
    &= R(K_{Rv, r-w\cdot v})+ w,
    \end{align*}
    we have
    \begin{align*}
    \mcal{F}_{RK+w}(v, r) &= \chi((RK+w)_{v, r})\\
    &= \chi( R(K_{Rv, r-w\cdot v})+  w )\\
    &=\chi(K_{Rv, r-w\cdot v})\\
    &= \mcal{F}_K(Rv, r- w\cdot v).
    \end{align*}
\end{proof}

\subsection{Proof of Proposition \ref{prop: Wasserstein stable}}
\begin{proof}
    We mainly follow the proof of Lemma 2.1 in \cite{turner2014persistent}.
    Let us look at two height functions, $h_{v_1}$ and $h_{v_2}$ in the direction $v_1$ and $v_2$ respectively defined on $K$.
    In other words, for any $x$ in $K$, $h_{v_i}(x)$ equals $x$ dot $v_i$ where $i\in \{1, 2\}$.
    If we set $L:= \sup\{\|x\|_2 \mid x\in K\}$, then we have
    $$
    \left|h_{v_1}(x)-h_{v_2}(x)\right|=\left|x \cdot v_1-x \cdot v_2\right| \leq\|x\|_2\left\|v_1-v_2\right\|_2 \leq L\left\|v_1-v_2\right\|_2.
    $$
    From the bottleneck stability theorem \citep{oudot2017persistence,chazal2016structure}, we have 
    $$
    d_b(B_K^{v_1},B_K^{v_2}) \leq \| h_{v_1} - h_{v_2}\|_\infty
    $$
    where $d_b$ is the bottleneck distance between two persistence diagrams.
    Combining these two inequalities, we get
    $$
    d_b(B_K^{v_1},B_K^{v_2}) \leq \| h_{v_1} - h_{v_2}\|_\infty \leq L\|v_1 - v_2\|_2.
    $$
\end{proof}

\subsection{Proof of Proposition \ref{prop: diagram o3 and trans}}
\begin{proof}
    First, we prove $\mcal{G}_{RK}(v) = \mcal{G}_{K}(Rv)$.
    Note that $\mcal{G}_{RK}(v)$ is the persistence diagram induced from the filtration $((RK)_{v, r})_{r\in\R} = (r\mapsto ((RK)_{v, r}))$.
    Since we have
    \begin{align*}
       (RK)_{v, r} &= \{x\in \R^n\mid x\in RK \text{ and } x\cdot v \leq r\}\\
       &= \{x\in \R^n \mid Rx \in K \text{ and } Rx\cdot Rv\leq r\}\\
       &= \{R^{-1}x\in \R^n \mid x \in K \text{ and } x\cdot Rv\leq r\} = R(K_{Rv, r}),
    \end{align*}
    we only need to prove that the persistence diagram induced from the filtration $(R(K_{Rv, r}))_{r\in \R}$ is the same as $\mcal{G}_K(Rv)$.
    Note that the filtrations $(R(K_{Rv, r}))_{r\in \R}$ and $(K_{Rv, r})_{r\in \R}$ induce the same persistence diagram.
    Since $\mcal{G}_K(Rv)$ is the persistence diagram induced from the filtration $(K_{Rv, r})_{r\in \R}$, the claim holds.

    Similarly, $\mcal{G}_{K+w}(v)$ is the persistence diagram induced from the filtration $((K+w)_{v, r})_{r\in \R}$.
    Since we have
    \begin{align*}
       (K+w)_{v, r} &= \{x\in \R^n\mid x\in K + w \text{ and } x\cdot v \leq r\}\\
       &= \{x\in \R^n \mid x - w \in K \text{ and } (x-w)\cdot v\leq r - v\cdot w\}\\
       &= \{x + w\in \R^n \mid x \in K \text{ and } x\cdot v\leq r - v\cdot w\} = w + K_{v, r - v\cdot w},
    \end{align*}
    and the filtrations $(w + K_{v, r - v\cdot w})_{r\in \R}$ and $(K_{v, r - v\cdot w})_{r\in \R}$ induce the same persistence diagram, the only difference between $\mcal{G}_{K+w}(v)$ and $\mcal{G}_K(V)$ is translation difference by $v\cdot w$.
\end{proof}

\subsection{Proof of Proposition \ref{prop: landscape o3 and trans}}
\begin{proof}
    We first show the $\on{O}(n)$-equivariance.
    As in the proof of Proposition \ref{prop: diagram o3 and trans}, we have 
    $(RK)_{v, r} = R(K_{Rv, r})$, and the filtrations $(R(K_{Rv, r}))_{r\in \R}$ and $(K_{Rv, r})_{r\in \R}$ induce the same persistence landscape.
    Since $\lambda_{v, c}^{R K}$ is the persistence landscape induced from $((RK)_{v, r})_{r\in \R}$ and $\lambda_{R v, c}^K$ is the persistence landscape induced from $(K_{Rv, r})_{r\in \R}$, the claim holds.

    For the second claim, we note that from the proof of Proposition \ref{prop: diagram o3 and trans}, we have 
    $$(K+w)_{v, r} = w + K_{v, r-v\cdot w}.$$
    From the equations
    \begin{align*}
        \beta_{K+w, v}^{t-m, t+m} :&= \dim(\on{Im}((K+w)_{v, t-m}\rightarrow(K+w)_{v, t+m}))\\
        &=\dim(\on{Im}(w + K_{v, t-m - v\cdot w}\rightarrow w + K_{v, t+m - v\cdot w}))\\
        &=\dim(\on{Im}(K_{v, t-m - v\cdot w}\rightarrow K_{v, t+m - v\cdot w}))\\
        &=\beta_{K, v}^{(t-v\cdot w) - m, (t- v\cdot w) + m},
    \end{align*}
    we have 
    \begin{align*}
        \lambda_{v, c}^{K+w}(t) &=\sup\{m\geq 0 \mid \beta_{K+w, v}^{t-m, t+m}\geq c\}\\
        &=\sup\{m\geq 0 \mid \beta_{K, v}^{(t-v\cdot w)-m, (t - v\cdot w)+m}\geq c\} = \lambda_{v, c}^K(t - v\cdot w).
    \end{align*}
\end{proof}

\subsection{Proof of Theorem \ref{thm: GNN equiv}}
We define an operator $f\diamond g:S^2\rr \R^\ell$:
$$
(f\diamond g)(x) = \int_{S^2} g(x -y) f(y) \; dV(y)
$$
where $dV$ is the volume form of $S^2$ satisfying $\int_{S^2}dV = 1$.
Then, the operator $\diamond$ is $\on{O}(3)$-equivariant:
\begin{lemma}\label{lem: integral equiv}
    $(Rf)\diamond g = R(f\diamond g)$ for every $R\in\on{O}(3)$.
\end{lemma}
\begin{proof}
    If $R$ is an element of $\on{O}(3)$, then 
    \begin{align*}
        ((Rf)\diamond g)(x) &= \int_{S^2} g(x - y)(Rf)(y)\; dV(y) =\int_{S^2} g(x - y)f(Ry)\; dV(y) \\
        &=\int_{S^2} g(x - R^{-1}y)f(y)\; dV(y) =\int_{S^2} g(Rx - y)f(y)\; dV(y)\\
        &= (f\diamond g)(Rx) = (R(f\diamond g))(x) 
    \end{align*}
    since $|\det(R)| = 1$ and $\|x - R^{-1}y\| =\|Rx - y\|$ for every $x, y \in S^2$.
\end{proof}

Now, we provide a proof of Theorem \ref{thm: GNN equiv} as follows:

\begin{proof}
    First we note that $\mbb{P}\left[
        \| R\mcal{T}(f)^\prime(x) - \mcal{T}(Rf)^\prime(x)\|_\infty > \epsilon
    \right]$ 
    is less than or equal to
    $$
    \mbb{P}\left[
        \| R\mcal{T}(f)^\prime(x) - (R(f\diamond g))(x)\|_\infty > \epsilon/2
    \right] + 
    \mbb{P}\left[
        \| ((Rf)\diamond g)(x) - \mcal{T}(Rf)^\prime(x)\|_\infty > \epsilon/2
    \right]
    $$
    by Lemma \ref{lem: integral equiv} the triangular inequality.
    Therefore, it is sufficient to show that 
    $$
    \mbb{P}\left[
        \| (f\diamond g)(x) - \mcal{T}(f)^\prime(x)\|_\infty > \epsilon/2
    \right]\lrr 0
    $$
    since 
    $\mbb{P}\left[ \| R\mcal{T}(f)^\prime(x) - (R(f\diamond g))(x)\|_\infty > \epsilon/2 \right] = \mbb{P}\left[ \| \mcal{T}(f)^\prime(Rx) - (f\diamond g)(Rx)\|_\infty > \epsilon/2 \right]$
    and 
    $ \mbb{P}\left[ \| ((Rf)\diamond g)(x) - \mcal{T}(Rf)^\prime(x)\|_\infty > \epsilon/2 \right] $ goes to $0$ in the same way.
    Also, without loss of generality, we can assume $\ell = 1$, so that it is sufficient to show that
    $$
    \mbb{P}\left[
        | (f\diamond g)(x) - \mcal{T}(f)^\prime(x) | > \epsilon
    \right]\lrr 0.
    $$

    Since $\mbb{E}_y\left[ g(x - y)f(y) \right] = \int_{S^2} g(x-y)f(y)\;dV(y) = (f\diamond g)(x)$, by Hoeffding's inequality \citep{boucheron2013concentration}, the claim holds.
\end{proof}

\subsection{Proof of Proposition \ref{prop: inv and equiv}}
\begin{proof}
    Since we convert the embedded simplicial complex to the Euler curves, a subdivision of the embedded simplicial complex cannot affect the outcome.
    Assume that $R\in\on{O}(n)$ and $w\in \mbb{R}^n$. 
    Then, we have 
    \begin{align*}
        \mcal{DF}_{RK + w}(v) &= \{\mcal{D}_1\circ\mcal{F}_{RK+w}(v), \dots, \mcal{D}_m\circ\mcal{F}_{RK+w}(v)\}\\
        &=\{\mcal{D}_1\circ\mcal{F}_{RK}(v), \dots, \mcal{D}_m\circ\mcal{F}_{RK}(v)\}
    \end{align*}
    since $\mcal{F}_{RK+w}(v)(x) = \mcal{F}_{RK}(v)(x - v\cdot w)$ for every $x\in \mbb{R}$ from Proposition \ref{prop: ECT equiv}. 
    Furthermore, from the same proposition, since we have $\mcal{F}_{RK}(v) = \mcal{F}_{K}(Rv)$,
    \begin{align*}
        \mcal{DF}_{RK + w}(v) &= \{\mcal{D}_1\circ\mcal{F}_{K}(Rv), \dots, \mcal{D}_m\circ\mcal{F}_{K}(Rv)\}\\
        &=\mcal{DF}_{K}(Rv).
    \end{align*}
\end{proof}
    
\subsection{Proof of Proposition \ref{prop: inv and EQUIV PL version}}
\begin{proof}
    Essentially, the proof is the same as the proof of Proposition \ref{prop: inv and equiv}.
    For the same reason as in the proof of \ref{prop: inv and equiv}, $\mcal{E}\lambda_*$ is subdivision-invariant.
    
    From Proposition \ref{prop: landscape o3 and trans},
    \begin{align*}
        \mcal{E}\lambda_{RK}(v) &=\left\{\mathcal{E}_1 \circ \lambda_v^{RK}, \ldots, \mathcal{E}_m \circ \lambda_v^{RK}\right\} =\left\{\mathcal{E}_1 \circ \lambda_{Rv}^K, \ldots, \mathcal{E}_m \circ \lambda_{Rv}^K\right\} \\
        &=\mcal{E}\lambda_K(Rv) = (R(\mcal{E}\lambda_K))(v),
    \end{align*}
    and 
    $$
        \mcal{E}\lambda_{K + w}(v)=\left\{\mathcal{E}_1 \circ \lambda_v^{K + w}, \ldots, \mathcal{E}_m \circ \lambda_v^{K + w}\right\} =\left\{\mathcal{E}_1 \circ \lambda_{v}^K, \ldots, \mathcal{E}_m \circ \lambda_{v}^K\right\} =\mcal{E}\lambda_K(v)
    $$
    for $R\in \on{O}(n)$ and $w\in \R^n$ since 
    $$
    \lambda_v^{K+w}(x, c) = \lambda_{v, c}^{K+w}(x) = \lambda_{v, c}^K(t - v\cdot w).
    $$
\end{proof}

\vskip 0.2in
\bibliography{main}

\end{document}